\documentclass{lipics-v2019}

\title{{\bf Situation Calculus by Term Rewriting}}

\author{David A. Plaisted}{Department of Computer Science, UNC Chapel Hill, Chapel Hill, NC 27599-3175, U.S.A.,Phone: (919) 590-6051}{plaisted@cs.unc.edu}{}{}

\authorrunning{D.\,A. Plaisted}

\Copyright{David Plaisted}



\EventEditors{John Q. Open and Joan R. Access}
\EventNoEds{2}
\EventLongTitle{42nd Conference on Very Important Topics (CVIT 2016)}
\EventShortTitle{CVIT 2016}
\EventAcronym{CVIT}
\EventYear{2016}
\EventDate{December 24--27, 2016}
\EventLocation{Little Whinging, United Kingdom}
\EventLogo{}
\SeriesVolume{42}
\ArticleNo{23}

\begin{document}

\maketitle
\begin{abstract}
A version of the situation calculus in which situations are represented as first-order terms is presented.  Fluents can be computed from the term structure, and actions on the situations correspond to rewrite rules on the terms.  Actions that only depend on or influence a subset of the fluents can be described as rewrite rules that operate on subterms of the terms in some cases.  If actions are bidirectional then efficient completion methods can be used to solve planning problems.  This representation for situations and actions is most similar to the fluent calculus of Thielscher \cite{Thielscher98}, except that this representation is more flexible and more use is made of the subterm structure.  Some examples are given, and a few general methods for constructing such sets of rewrite rules are presented.
\end{abstract}

\section{Introduction}
The situation calculus permits reasoning about properties of
situations that result from a given situation by sequences of actions
 \cite{McHa:69}.  In the situation calculus, situations (states) are represented
explicitly by variables, and actions $a$ map states $s$ to states $do(a,s)$.
Predicates and functions on a situation or state are called {\em fluents}.  In
some formalisms, a situation denotes a state of the world, specifying
the values of fluents, so that two situations are equal if the values
of all their fluents are the same.  Other formalisms reserve the term
situation for a sequence of states.  A problem with the situation
calculus or any formalism for reasoning about actions is the necessity
to include a large number of {\em frame axioms} that express the fact
that actions do not influence many properties (fluents) of a
state.  Since the early days of
artificial intelligence research the frame problem has been studied,
beginning with McCarthy and Hayes \cite{McHa:69}.  Lin \cite{Lin:08}
has written a recent survey of the situation calculus.

Reiter \cite{Reiter:91} proposed an approach
to the frame problem in first-order logic that avoids the need to
specify all of the frame axioms.  The method of Reiter, foreshadowed
by Haas \cite{Haas:87}, Pednault \cite{Pednault:89}, Schubert
 \cite{schubert90monotonic} and Davis \cite{Davis:02}, essentially
solves the frame problem by specifying that a change in the truth
value of a fluent, caused by an action, is equivalent to a certain
condition on the action.  In this formalism, it is only necessary to
list the actions that change each fluent, and it is not necessary to
specify the frame axioms directly.  If an action does not satisfy the
condition, the fluent is not affected.  In the following discussion the term ``Reiter's
formalism'' will be used for simplicity even though others have also
contributed to its development.  The {\em fluent calculus} \cite{Thielscher98} is another
interesting approach to the frame problem.  In this approach, a state is a conjunction of known facts

Petrick \cite{Petrick08} has adapted Reiter's formalism to knowledge and belief and has also introduced the notion of a
Cartesian situation that can decompose a situation into parts, in a way that appears to be similar to the aspect calculus.  However,
his formalism also considers a situation to include a sequence of states.

There are some problems with Reiter's formalism, especially in its suitability for first-order theorem provers.  In Reiter's formalism, the successor state axiom for a fluent essentially says that the fluent is true on a situation $do(a,s)$ for fluent $a$ and situation $s$ if $a$ is an action
that makes the fluent true, or if the fluent was already true and $a$ is not one of the actions that makes the fluent false.  This requires one to know under what conditions an action changes
the value of the fluent to "true" or "false."  If for example the action is nondeterministic this may be
difficult to know.  Also, to formulate the successor state axiom, one needs a theory of equality
between actions.  If there are only a small number of actions that
can make a fluent false, then Reiter's formalism is concise because one need not list all of the actions that do not influence the fluent (the frame axioms for the fluent).  However, if there are many 
actions (possibly thousands or millions) that influence the fluent, then this successor state axiom can become very long.  Further, when converting Reiter's approach to clause form, one needs an axiom of the form
"For all actions $a$, $a = a_1 \vee a = a_2 \vee \dots \vee a = a_n$" where $a_i$ are all the possible actions, as well as the axioms $a_i \neq a_j$ for all $i \neq j$.  If there
are many actions, the first axiom will be huge.  It is also difficult for many theorem provers to handle axioms of this form.

Even the successor state axiom itself, when
translated into clause form, produces clauses having a disjunction of an equation and another literal.  Using $\Phi(p,s)$ to denote the value of fluent $p$ on situation $s$, a simple form of the successor state axiom would be
\[\Phi(p, do(x,s)) \equiv [(\Phi(p, s) \wedge (x \neq a_1) \wedge (x \ne a_2)) \vee
(x  = b_1 \vee x = b_2)]\]
where $a_1$ and $a_2$ are the only actions that can make $p$ false and $b_1$ and $b_2$ are the
only actions that make
$p$ true.  Consider an even simpler form:
\[\Phi(p, do(x,s)) \equiv [(\Phi(p, s) \wedge (x \neq a_1) ) \vee
(x  = b_1 )]\]
The clause form of the latter is $\neg \Phi(p, do(x,s)) \vee \Phi(p, s) \vee x = b_1, \neg \Phi(p, do(x,s)) \vee x \neq a_1 \vee x = b_1, x \neq b_1 \vee \Phi(p, do(x,s)), \neg \Phi(p, s) \vee x = a_1 \vee \Phi(p, do(x,s))$.  Such conjunctions of equations and inequations can be difficult for theorem provers to handle, especially if there are more actions in which case there would be more equations and inequations in the clauses.

\section{Underlying Theory}
We assume that there is some underlying set ${\cal U}$  of axioms in first-order logic concerning states, fluents, and
actions. The semantics of this axiomatization will have domains for states and actions, with fluents mapping from states to various domains.   

Actions in ${\cal U}$ are typically indicated by the letter $a$, possibly with subscripts, and fluents are typically indicated by the letters $p$ and $q$, possibly with subscripts.  ${\cal F}$ is the set of all fluents and ${\cal A}$ is the set of actions.  States are denoted by
$s'$, $t'$, and $u'$,  possibly with subscripts.  The set of states is ${\cal S}$.

If $a$ is an action and $s'$ is a state then $do(a,s')$ is the result of applying action $a$ in state $s'$.  If $p$ is a fluent then $\Phi(p,s')$ is the value of $p$ on state $s'$.  Thus fluents are
essentially functions from states to various domains.  If the value of a fluent is $true$ or $false$, and it is not parameterized, then $\Phi(p,s')$ may be written as $p(s')$ instead.  The semantics (interpretation) of the underlying theory ${\cal T}$ is assumed to have sorts for fluents, states, and actions, in addition to possibly others.

\section{Syntax}
Term rewriting systems\cite{bani:98} have a simple syntax and semantics.  In this paper a situation calculus based on first-order term
rewriting is presented.  Situations are represented by terms and actions are represented by rewrite rules that operate on terms.  This
representation permits completion procedures \cite{} to be used for planning if actions are bidirectional, and also makes use of the subterm
structure of terms to separate actions whose effects are independent.  This can improve the efficiency of the planning process.  This
representation is probably most closely related to the fluent calculus of Thielscher \cite{} among the approaches that have been proposed to date.

The basics of term rewriting systems are as follows.
\subsection{Terms}
The symbols $f, g, h$ are {\em function symbols}, $x, y, z$ are {\em variables}, $r, s, t, u, v, w$ are {\em terms}, and $a, b, c$  are {\em individual constants}.  Also, $i,j,k$ are variables that are intended to denote integers.  $F$ is the set of function symbols and
$X$ is the set of variables.

The {\em arity} of a function symbol is the number of arguments it takes.  We assume there is a bound on the maximum arity of any function symbol.  Terms are defined as follows:  A variable or an individual constant is a term.  Also, if $f$ has arity $n$ and $t_1, \dots, t_n$ are
terms then $f(t_1, \dots, t_n)$ is a term.  
The set of terms over a set $F$ of function symbols and a set
$X$ of variables is denoted $T[F,X]$.
A term is a {\em ground term} if it contains no variables.  The notation $s \equiv t$ for terms $s$ and $t$ indicates that the terms are syntactically identical. A term $s$ is a
{\em subterm} of $f(t_1, \dots, t_n)$ if $s \equiv f(t_1, \dots, t_n)$ or $s$ is a subterm of $t_i$ for some $i$.  Also,
$s$ is a {\em proper} subterm of $f(t_1, \dots, t_n)$ if $s$ is a subterm of $t_i$ for some $i$.  A {\em maximal term} in a set $T$ of terms is a term in $t$ that is not a proper subterm of any other term in $T$.  The size $|s|$ of a term $s$ is
defined as follows:  $|x| = |c| = 1$ for variables $x$ and individual constants $c$.  Also, $|f(t_1, \dots, t_n)| =
1 + |t_1| + \dots + |t_n|$.

A {\em context} is a term with one occurrence of $\Box$ in it, such as $f(a, \Box, b)$.  This is written as $t[~]$ and the result of substituting some term $u$ for $\Box$ is written as $t[u]$.  This notation can be extended to $t[u_1, \dots, u_n]$ indicating specific occurrences of the subterms $u_1, \dots, u_n$.

A {\em substitution} is a replacement of variables $x_i$ by terms $t_i$; this can be written as $\{x_1 \leftarrow t_1, \dots, x_n \leftarrow t_n\}$ or as $\{t_1/x_1, ... t_n/x_n\}$.  Also, one can write $t\{t_i/x_i\}$ as $t (t_i/x_i)$.  Greek symbols such as $\theta$ are commonly used for substitutions.  The result of applying a substution $\theta$ to a term $t$ is written as $t \theta$.  The term $t\theta$ is called an {\em instance} of $t$.  Two  terms $r$ and $s$ are {\em unifiable} if they have a common instance.

\subsection{Term Rewriting Systems}
A {\em rewrite rule} is of the form $r \rightarrow s$ where $r$ and $s$ are terms and all variables in $s$ also occur in $r$.  A {\em term rewriting system} is a finite or infinite set of rewrite rules.  The relation $\Rightarrow$ is defined by $t_1 \Rightarrow_R t_2$ iff there is a context $t[~]$ such that $t_1 \equiv t[r\theta]$ and $t_2 \equiv t[s\theta]$ for some rewrite rule $r \rightarrow s$ in $R$ and some substitution $\theta$.  For example, if the rule $f(g(x)) \rightarrow f(h(x,x))$ is in $R$, then $f(f(g(h(a,b)))) \Rightarrow_R f(f(h(h(a,b),h(a,b))))$.  The subterm occurrence $r\Theta$ is called a {\em redex}.  Also, $\Rightarrow_R^*$ is the (reflexive) transitive closure of $\Rightarrow_R$.  A sequence $t_1 \Rightarrow_R t_2 \Rightarrow_R t_3 \dots$ is called a {\em rewrite sequence}.  The system $R$ is {\em terminating} if it has no infinite rewrite sequences.  A term $s$ is {\em reducible} for $R$ if there is a term $t$ such that $s \Rightarrow_R t$; otherwise $s$ is {\em irreducible}.  If $s \Rightarrow_R^* t$ and $t$ is irreducible then one writes $s \Rightarrow !_R t$ and $t$ is called a {\em normal form} of $s$.  A term rewriting system $R$ is {\em confluent} if for all terms $s, t_1$, and $t_2$, if $s \Rightarrow_R^* t_1$ and $s \Rightarrow_R^* t_2$ then there is a term $u$ such that $t_1 \Rightarrow_R^* u$ and $t_2 \Rightarrow_R^* u$.
$R$ is {\em bidirectional} or {\em invertible} if $s \Rightarrow_R t$ implies
$t \Rightarrow_R s$ for all terms $s,t$ in $T[F,X]$.

\section{Terminology}

Instead of situations, we shall actually refer to {\em states}. Recall from above that the symbols $s', t', u'$ refer to states and the set of states is ${\cal S}$.  The set of actions is ${\cal A}$ and $a$, $b$ refer to actions.  If  $s'$ is
a state and $a$ is an action then $do(a,s')$ is a state obtained by performing action $a$ on state $s'$.  There are also {\em fluents},
which map states onto various domains; ${\cal F}$ is the set of fluents and the symbols $p,q$ refer to fluents.  If $s'$ is a state and $p$ is a fluent then $\Phi(p,s')$ is
the value of fluent $p$ on state $s'$.  We assume that ${\cal U}$ satisfies the fluent
dependence condition:

\begin{definition}
\label{fluent.dependence.condition}
The {\em fluent dependence condition} is the following: For
all states $s'$ and $t'$,
if $\Phi(p,s') = \Phi(p,t')$ for all fluents $p$ then $do(a,s') = do(a,t')$ so that the effect
of an action depends only on the fluents of the state.  In fact, it may be best to assume
that if $\Phi(p,s') = \Phi(p,t')$ for all fluents $p$ then $s' = t'$.
\end{definition}

There may be some
combinations of fluents that do not correspond to states in the underlying theory $\cal U$.
Thus there may be a {\em fluent constraint} that specifies which combinations of fluents
correspond to states in ${\cal S}$.  In our examples the fluent constraint is typically always
satisfied so that all combinations of fluents correspond to states in ${\cal U}$.

\begin{definition}
If $\{(p_1,v_1), \cdots, (p_n,v_n)\}$ is a set of fluents together with their possible
values, then $\chi(\{(p_1,v_1), \cdots, (p_n,v_n)\})$ is true if there is a state
$s' \in {\cal S}$ such that $\Phi(p_i,s') = v_i$ for $1 \le i \le n$; otherwise it is false.
\end{definition}

Based on ${\cal U}$ we construct a term-rewriting system $R_{\cal U}$ to simulate
the actions of ${\cal U}$ using rewrite rules.

\begin{definition}
Let $T[F,X]$ be the set of first-order terms over a finite set $F$ of {\em function symbols} and a set $X$ of {\em variables}.  The symbols $r,s,t,u,v,w$ will represent terms.
\end{definition}

\begin{definition}
\label{sigma.definition}
Let ${\cal T} \subseteq T[F,X]$ be a set of first-order terms that represent states
in ${\cal U}$ and suppose $\sigma$ is a function mapping such terms to
states:  $\sigma : {\cal T} \rightarrow {\cal S}$.  Thus for all $t \in {\cal T}$,
$\sigma(t)$ satisfies the fluent constraint in ${\cal U}$.  For now we assume all terms in ${\cal T}$
are ground terms so that for all $t \in {\cal T}$, for all fluents $p$, $\Phi(p,\sigma(t))$ is
defined.
\end{definition}

We assume the following:

\begin{equation}
\label{term.set.definition}
\mbox{For all $s' \in {\cal S}$ there exists a ground term $s \in {\cal T}$ such that $\sigma(s) = s'$.}
\end{equation}

\begin{definition}
\label{hat.Phi.definition}
Let $\hat{\Phi}(p,t)$ be a procedure in $R_{\cal U}$ such that for all fluents $p$ and terms $t$,
$\hat{\Phi}(p,t) = \Phi(p,\sigma(t))$.  We assume $\hat{\Phi}(p,t)$ is computable as a
function of $t$, that is, the fluents can be determined from the term structure of $t$.
Typically in our examples $\hat{\Phi}$ is easy to compute.
\end{definition}

\subsection{Rewrite Rules for Situation Calculus}

\begin{definition}
\label{R sub U definition}
There are two kinds of rewrite rules in $R_{\cal U}$:  {\em rearrangement
rules} that don't affect the state and {\em action rules} that simulate actions on states.  $E_{\sigma}$ is the set of
rearrangement rules and $A_{\sigma}$ is the set of action rules.  It is assumed that rules in $E_{\sigma}$ are invertible.  The rearrangement rules reformat the term without affecting the state; for example, they may permit a list of terms to be sorted in an arbitrary order.
\end{definition}

Although ${\cal T}$ is a set of ground terms, the rewrite rules in $R_{\cal U}$ need not be
ground rules, which will be clear from the examples.  These rules are assumed to satisfy the following axioms:

\begin{eqnarray}
\label{state.closure.equation}
&\mbox{If $\alpha \in E_{\sigma} \cup A_{\sigma}$ and $s \in {\cal T}$ and $t$ is a term such that $s \Rightarrow_{\alpha} t$} \nonumber \\ 
&\mbox{then $t \in {\cal T}$ (so that $\sigma(t)$ also satisfies the fluent
constraint for ${\cal U}$)}.
\end{eqnarray}
\begin{equation}
\label{rearrangements.to.states.equation}
\mbox{For terms $s,t \in {\cal T}$, (there exists $\alpha \in E_{\sigma}$ such that $s \Rightarrow_{\alpha} t$) iff $\sigma(s) = \sigma(t)$.}
\end{equation}
Rules in $A_{\sigma}$ may correspond to more than one action in ${\cal A}$.  Thus we have
the following axiom:
\begin{eqnarray}
\label{rule.to.action.equation}
&\mbox{If $\alpha \in A_{\sigma}$  then for all $s,t \in {\cal T}$, if $s \Rightarrow_{\alpha} t$} \nonumber \\
&\mbox{then there exists $a \in {\cal A}$ such that $\sigma(t) = do(a,\sigma(s))$.}
\end{eqnarray}
In fact we assume that the action $a$ is computable, given $s$, $t$, and $\alpha \in A_{\sigma}$.

We also allow the possibility that actions in ${\cal A}$ correspond to more than one rule in
$A_{\sigma}$.

\begin{eqnarray}
\label{do.rewrite.equation}
&\mbox{For all $s' \in {\cal S}$ and all $a \in {\cal A}$ and for all $s,t \in {\cal T}$, if $\sigma(s) = s'$ and $\sigma(t) = do(a,\sigma(s'))$} \nonumber \\
&\mbox{then there are terms $u,v$ such that $s \Rightarrow^*_{E_{\sigma}} u \Rightarrow_{A_{\sigma}} v \Rightarrow^*_{E_{\sigma}} t$.}
\end{eqnarray}

\begin{definition}
A term-rewriting system $R_{\cal U}$ {\em represents} the theory ${\cal U}$ if there is a set ${\cal T}$ of terms satisfying
equation \ref{term.set.definition} and if
$R_{\cal U} = E_{\sigma} \cup A_{\sigma}$ where $E_{\sigma}$ and $A_{\sigma}$ are sets and there are functions $\sigma$ and $\hat{\Phi}$ such that $E_{\sigma}$ and $A_{\sigma}$ and the functions $\sigma$ and $\hat{\Phi}$ satisfy
Definitions \ref{sigma.definition}, \ref{hat.Phi.definition}, and \ref{R sub U definition}, and Equations \ref{state.closure.equation}, \ref{rearrangements.to.states.equation}, \ref{rule.to.action.equation}, and \ref{do.rewrite.equation}.
\end{definition} 

\subsection{Planning Using Term Rewriting}

\begin{definition}
\label{stackrel.def}
For $s,t \in {\cal T}$, $s \stackrel \sim \Rightarrow t$ if there are terms $u,v \in {\cal T}$ such
that $s \Rightarrow^*_{E_{\sigma}} u \Rightarrow_{A_{\sigma}} v \Rightarrow^*_{E_{\sigma}} t$.
\end{definition}

\begin{theorem}[Planning Theorem]
Suppose $s_1$ and $s_n$ are terms and there are states $s'_1, s'_2, \cdots, s'_n$ and actions
$a_1, a_2, \cdots, a_{n-1}$ such that
$s'_1 = \sigma(s_1), s'_n = \sigma(s_n)$, and for all $i$, $1 \le i \le n-1$, $s'_{i+1} = do(a_i,s_i)$.
Then $s_1 {\stackrel \sim \Rightarrow}^* s_n$.  The converse is also true.
\end{theorem}

\begin{proof}
By Definition \ref{stackrel.def} and Equation \ref{do.rewrite.equation}.  The converse
follows from Equations \ref{rearrangements.to.states.equation} and \ref{rule.to.action.equation}.
\end{proof}
Thus to see if it is possible to reach $s'_n$ from $s'_1$ by a sequence of actions in ${\cal U}$, one can construct terms $s_1$ and $s_n$
as in the theorem and test if $s_1 {\stackrel \sim \Rightarrow}^* s_n$.

\begin{corollary}
With conditions as in the planning theorem, if $A_{\sigma}$ is invertible,
then $s_1 {\stackrel \sim \Leftrightarrow}^* s_n$.
\end{corollary}

In this case, rewrite strategies such as completion and unfailing completion \cite{bani:98} can be used
to test if $s_1 {\stackrel \sim \Leftrightarrow}^* s_n$.  If it is shown that $s_1 {\stackrel \sim \Leftrightarrow}^* s_n$ then a plan (sequence of actions) can be extracted from
the proof.  This plan may not be optimal, but it may be possible to optimize it after it is
found.

To extract a plan from a proof that $s_1 {\stackrel \sim \Leftrightarrow}^* s_n$, the
notation $s {\stackrel [{R^*_{\cal U}}]{\alpha_1 \cdots \alpha_k} {\Rightarrow}} t$
can be used indicating a sequence of actions leading from $\sigma(s)$ to $\sigma(t)$.
These actions can be carried along in the completion procedure.  We have the following
rules:

\begin{equation}
\mbox{If $r {\stackrel [{R^*_{\cal U}}]{\alpha_1 \cdots \alpha_m} {\Rightarrow}} s$ and
$s {\stackrel [{R^*_{\cal U}}]{\alpha_{m+1} \cdots \alpha_n} {\Rightarrow}} t$
then $r {\stackrel [{R^*_{\cal U}}]{\alpha_1 \cdots \alpha_n} {\Rightarrow}} t$.} 
\end{equation}

\begin{equation}
\mbox{If $s {\stackrel [{R^*_{\cal U}}]{\alpha_1 \cdots \alpha_k} {\Rightarrow}} t$
then $t {\stackrel [{R^*_{\cal U}}]{\alpha_k \cdots \alpha_2 \alpha_1} {\Rightarrow}} s$
assuming $A_{\sigma}$ is invertible.}
\end{equation}

\begin{equation}
\mbox{$s {\stackrel [{E^*_{\sigma}}]{\epsilon}{\Rightarrow}} t$ for all $s,t$ in $\cal T$ where $\epsilon$ is the empty sequence.}
\end{equation}

\begin{equation}
\mbox{
If $\sigma(t) = do(a,\sigma(s))$ and $s \Rightarrow_{A_{\sigma}} t$ then $s {\stackrel [{R^*_{\cal U}}]{a}{\Rightarrow}} t$.}
\end{equation}

However, there are some problems with this approach. One problem is that the same rewrite relation between $s$ and $t$ may be
derived with more than one action sequence, complicating the search.  Also, there may be more than one action corresponding to
a given rewrite rule in $A_{\sigma}$.

In fact, it's not necessary to carry along sequences of actions.  From a proof that $s_1 {\stackrel \sim \Leftrightarrow}^* s_n$ it
is possible to construct a sequence $s_1, s_2, \cdots, s_n$ of ground terms such that $s_i \Rightarrow_{R_{\cal U}} s_{i+1}$,
$1 \le i \le n-1$.  Then by repeated use of the decidability result following Equation \ref{rule.to.action.equation}, one can find
a sequence $a_1, a_2, \cdots, a_{n-1}$ of actions such that $\sigma(s_{i+1}) = do(a_i, \sigma(s_i))$ for all $i$, $1 \le i \le n-1$.

\begin{definition}
The {\em constraint predicate} $\hat{\chi}$ is defined so that $\hat{\chi}(t)$ is true for $t \in T[F,X]$ if
$\chi(\{(p,\Phi(p,$ $\sigma(t))) : p \in {\cal F}\})$ is true, otherwise $\hat{\chi}(t)$ is false.
\end{definition}

\begin{definition}
Let $Find$ be a function which, given a set $\{(p_1,v_1), \cdots, (p_n,v_n)\}$ of
fluents and their values, finds a term $t$ such that $\Phi(p_i,\sigma(t)) = v_i$ for all $i$
if such a term exists, else returns "none."  We assume that the function $Find$ is computable.  There may be constraints on the fluents so that some
combinations of fluents do not correspond to any state in ${\cal S}$ because of the fluent
constraint for ${\cal U}$.
\end{definition}

The procedure $Find$ can be used to find the terms $s_1$ and $s_n$ in the planning
theorem if the values of all fluents on $s'_1$ and $s'_n$ are known.  The constraint predicate $\hat{\chi}$ needs to be incorporated into planning if $\chi$ is
not always true.  However, because of Equation \ref{state.closure.equation}, if one finds a sequence  $s_1, s_2, \cdots, s_n$ of ground terms such that $s_i \Rightarrow_{R_{\cal U}} s_{i+1}$,$1 \le i \le n-1$, then all the terms $s_i$, $1 \le i \le n$ satisfy the constraint $\hat{\chi}$.

So far the effect of rewrite rules operating on proper subterms of a term $t$ in ${\cal T}$
has not been discussed, but this can be significant.

\section{Examples}
Some examples will illustrate the properties of this approach to the situation calculus.
In these examples, the underlying theory ${\cal U}$ is described informally.  The
approach used for rewriting is unfailing completion \cite{bani:98}, which in the limit
produces a term rewriting system that is confluent by ordered rewriting.  we assume
that $R_{\cal U}$ is bidirectional.  Given
a term $s$ representing a starting state and a term $t$ representing a goal state,
both $s$ and $t$ are rewritten to a common term $u$ using the limiting
rewriting system and ordered rewriting.  Then the plan to get from
$s$ to $t$ is obtained from the rewrite sequence $s \Rightarrow^* u \Leftarrow^* t$.  This approach typically produces plans quickly
but they are not always optimal.  Frequently the plan produced can be made shorter by local
optimizations.

\subsection{Switches Example}

In this example, there are $n$ switches which can be on or off.  The state of the switches
is represented by a term in ${\cal T}$ of the form $f(x_1, x_2, \cdots, x_n)$ where all $x_i$ can be ``on" or ``off".
Each switch can be turned on or off.  So there is a rewrite rule in $R_{\cal U}$
\[f(x_1, \cdots, x_{i-1}, \mbox{off}, x_{i+1}, \cdots, x_n) \rightarrow f(x_1, \cdots, x_{i-1}, \mbox{on}, x_{i+1}, \cdots, x_n)\]
to turn the $i^{th}$ switch on and a rule
\[f(x_1, \cdots, x_{i-1}, \mbox{on}, x_{i+1}, \cdots, x_n) \rightarrow f(x_1, \cdots, x_{i-1}, \mbox{off}, x_{i+1}, \cdots, x_n)\]
to turn it off.

This set of rules is already bidirectional.
Now, suppose the problem is to get from $f(x_1, \cdots, x_n)$ to
$f(y_1, \cdots, y_n)$.  How will this be solved using completion and term rewriting?

The two rewrite rules will be oriented in one direction, depending on the termination ordering.  Suppose that ``off" is larger than ``on" in this ordering.  Then both
$f(x_1, \cdots, x_n)$ and $f(y_1, \cdots, y_n)$ will rewrite to $f(\mbox{on}, \cdots, \mbox{on})$, leading to a rewriting sequence $f(x_1, \cdots, x_n) \Rightarrow^*
f(\mbox{on}, \cdots, \mbox{on}) \Leftarrow^* f(y_1, \cdots, y_n)$,
so that the plan will be to turn all the switches $x_i$ off that are on, and
then turn on all the switches $y_i$ that are off.  This could result in a switch that is on
in both starting and ending states to be turned off and then on again.  This plan can be
optimized by removing such pairs of actions, resulting in a reasonable plan.

In this example, all actions can be encompassed in the single rules
$\mbox{off} \rightarrow \mbox{on}$ and $\mbox{on} \rightarrow \mbox{off}$.
If these rules are used, then the same plan will be derived as before.  Notice here
that a single rewrite rule corresponds to multiple actions.

Consider the effect of permitting rules in $R_{\cal U}$ to be non-ground and to rewrite on
proper subterms of terms in ${\cal T}$.  If the rules had to be ground rules and rewrite the
whole term, then they would be of the form $f(a_1, \cdots, a_n) \rightarrow
f(b_1, \cdots, b_n)$ where $a_i, b_i \in \{\mbox{off}, \mbox{on}\}$ and one of the $b_i$ differs from
$a_i$.  There would be $n * 2^n$ rules.  By allowing the rules to be non-ground there
are $2n$ rules of the form $f(x_1, \cdots, a_i, \cdots, x_n) \rightarrow
f(x_1, \cdots, b_i, \cdots, x_n)$ where $a_i, b_i \in \{\mbox{off}, \mbox{on}\}$  and $a_i \ne b_i$.
Each rule is of length about $n$ leading to an overall complexity that is quadratic.
If we now allow rewriting proper subterms, just the rules $\mbox{off} \rightarrow \mbox{on}$
and $\mbox{on} \rightarrow \mbox{off}$ suffice.  Similar comments about non-ground
rules and rewriting subterms apply to the following examples.

Now consider a slightly different example in which the state is represented as
$f(x_1, \cdots, x_n, z)$ where $z$ is $true$ if all switches are on and $false$ otherwise.
This introduces a constraint, and makes it more difficult to express the actions.  One
possibility is to have $2^n$ ground rewrite rules that specify the settings of all $n$ switches
and the effect of turning one switch on or off; this may or may not affect $z$.  Of course
it is better to avoid so many rules.

Another possibility is to have the terms in ${\cal T}$ only represent the state of the switches and
not the value of $z$.  This corresponds to a term in ${\cal T}$ only encoding a subset of the
fluents, if the other fluents are determined by this subset.

Still another possibility is to have rules of the form
\[f(x_1, \cdots, x_{i-1}, \mbox{on}, x_{i+1}, \cdots, x_n,z) \rightarrow f(x_1, \cdots, x_{i-1}, \mbox{off}, x_{i+1}, \cdots, x_n,\mbox{false})\]
to turn the $i^{th}$ switch off.  To turn the $i^{th}$ switch on one needs rules of
the form
\[f(x_1, \cdots, x_{i-1}, \mbox{off}, x_{i+1}, \cdots,\mbox{off}, \cdots, x_n,\mbox{false}) \rightarrow f(x_1, \cdots, x_{i-1}, \mbox{on}, x_{i+1}, \cdots, \mbox{off}, \cdots, x_n,\mbox{false})\]
if another switch is off, and rules of the form
\[f(\mbox{on}, \cdots, \mbox{on}, \mbox{off}, \mbox{on}, \cdots, \mbox{on},\mbox{false}) \rightarrow f(\mbox{on}, \cdots, \mbox{on}, \mbox{on}, \mbox{on}, \cdots, \mbox{on},\mbox{true})\]
if all other switches are on.

Finally, this theory could be represented by {\em constrained} rewrite rules of the
form $f(x_1, \cdots, \mbox{off},$ $\cdots, x_n, z) \rightarrow f(x_1, \cdots, \mbox{on}, \cdots, x_n, z')$ with the constraint that $z' = true$ iff all $x_i$ are on.

Another representation (for the problem without $z$) is to use a term $f(g_1(x_1),g_2(x_2), \cdots, g_n(x_n))$ to
represent the state of the switches, where the $g_i$ identify which switch is referred to
and $x_i$ gives its state, on or off.  This permits the rewrite rules to be of the form
$g_i(\mbox{off}) \rightarrow g_i(\mbox{on})$ and $g_i(\mbox{on}) \rightarrow g_i(\mbox{off})$, so that each rewrite rule refers to a different action.

\subsection{Tower of Hanoi}
For this example, there are $n$ disks of different sizes on three pegs.  On each peg the disks have to be in order of size, with the largest disk on the bottom.  Only the top
disk on a peg can be moved from one peg to another, and it has to be the smallest
disk on the peg it is moved to.  The problem is to rearrange the disks on the pegs, typically moving all of them from one peg to another.

The optimal sequence to move $n$ disks from peg $i$ to peg $j$, for $i \ne j$, consists
of moving the $n-1$ smallest disks from peg $i$ to peg $k$, $k \ne i$ and $k \ne j$,
then moving the largest disk from peg $i$ to peg $j$, then moving the $n-1$ smallest disks
from peg $k$ to peg $j$.  If $n=1$ this consists of one move, so the general sequence
has $2^n -1$ moves.

For this problem, terms in ${\cal T}$ can be of the form $f(x_1, x_2, \cdots, x_n)$ where each
$x_i$ is either 1, 2, or 3 depending on which peg the disks are on, and $x_1$ refers to the
largest peg, $x_2$ to the next largest, and so on.

The rules in $R_{\cal U}$ are of the form
\[f(i,j,j, \cdots, j) \rightarrow f(k,j,j, \cdots, j)\]
where $i,j$, and $k$ are distinct
elements of the set $\{1,2,3\}$ and also to move other disks, rules of the form
\[f(x_1, x_2, \cdots, i,j,j, \cdots, j) \rightarrow f(x_1,x_2,\cdots,k,j,j,\cdots,j)\]
again where $i,j$, and $k$ are distinct elements of $\{1,2,3\}$.  The $i^{th}$ largest
disk can only move from peg $i$ to peg $k$ if all the smaller disks are on peg $j$, because
otherwise, a smaller disk will be on peg $i$ or $k$, preventing the move.  For example,
consider the largest disk.  To move it from peg 1 to peg 2, there can't be any smaller
disks on peg 1 because they would be on top of it, and only the top disk can be moved.
There also can't be any smaller disks on peg 2, because then a larger disk would be put
on top of a smaller one, which is not permitted.  Now consider the smallest disk.  It is
always on top of one of the piles, and at any time it can move to any other pile.  In
general, a disk is only constrained in moving by disks that are smaller than it is.

Completing this set of rules generates rules of the form
\[f(x_1, x_2, \cdots, x_i, x_{i+1}, \cdots, x_n) \rightarrow f(x_1, x_2, \cdots, x_i, 1,1, \cdots, 1)\]
for all $j \in \{1,2,3\}$ assuming a lexicographic ordering with $3 > 2 > 1$.
Then to get from one arrangement of disks to another, a plan will be generated to
move all disks to peg 1, then move them back to the pegs in the goal situation.  This
is not optimal, but again, it may be possible to optimize this plan.

Another representation is to use the term $f(x_1, f(x_2, \cdots, f(x_n,\bot \cdots ))$
instead of $f(x_1, x_2, \cdots, x_n)$.  This permits rewrite rules to be applied to proper
subterms of the terms in ${\cal T}$.

For this problem, all actions are invertible.  Consider the system with three disks; then there
are the following rules:

\begin{center}
$f(1,\bot) \leftrightarrow f(2,\bot)$ \\
$f(1,\bot) \leftrightarrow f(3,\bot)$ \\
$f(2,\bot) \leftrightarrow f(3,\bot)$ \\
$f(1,f(3,\bot)) \leftrightarrow f(2,f(3,\bot))$  \\
$f(1,f(2,\bot)) \leftrightarrow f(3,f(2,\bot))$ \\
$f(2,f(1,\bot)) \leftrightarrow f(3,f(1,\bot))$ \\
$f(1,f(3,f(3,\bot))) \leftrightarrow f(2,f(3,f(3,\bot)))$  \\
$f(1,f(2,f(2,\bot))) \leftrightarrow f(3,f(2,f(2,\bot)))$  \\
$f(2,f(1,f(1,\bot))) \leftrightarrow f(3,f(1,f(1,\bot)))$  \\
\end{center}

The rules will be oriented $f(3,\bot) \rightarrow f(2,\bot) \rightarrow f(1,\bot)$.  Then
using these rules, the next three rules will become

\begin{center}
$f(1,f(1,\bot)) \leftrightarrow f(2,f(1,\bot))$  \\
$f(1,f(1,\bot)) \leftrightarrow f(3,f(1,\bot))$ \\
$f(2,f(1,\bot)) \leftrightarrow f(3,f(1,\bot))$ \\
\end{center}
which requires two rewrites for each rule, one on each side.  The rule
$f(1,f(1,\bot)) \leftrightarrow f(3,f(1,\bot))$ corresponds to the action sequence
``Move the small disk to peg 2, then move the next smallest disk to peg 3, then move
the small disk to peg 1."
Similarly, rewriting using these six rules, the last three rules become

\begin{center}
$f(1,f(1,f(1,\bot))) \leftrightarrow f(2,f(1,f(1,\bot)))$  \\
$f(1,f(1,f(1,\bot))) \leftrightarrow f(3,f(1,f(1,\bot)))$  \\
$f(2,f(1,f(1,\bot))) \leftrightarrow f(3,f(1,f(1,\bot)))$  \\
\end{center}
which requires four rewrites on each rule, two on each side.  In general completing the
system in this way requires a quadratic number of rewrite operations even though the
optimal action sequence requires an exponential number of actions.  Finally, rewriting
the starting term and the goal term to a common term requires a number of rewrites
that is linear in $n$.

If the problem is to show $f(2,f(2,f(2,\bot))) \leftrightarrow f(3,f(3,f(3,\bot)))$ then
both sides simplify to $f(1,f(1,f(1,\bot)))$.  The plan will then move all the disks from
peg 2 to peg 1, and then move them all from peg 1 to peg 3.  This is not optimal, but
perhaps it can be optimized by local transformations.  However, with a different
ordering, the plan can be improved.  If the ordering has $1 > 2 > 3$ then the
term $f(2,f(2,f(2,\bot)))$ will be rewritten to $f(3,f(3,f(3,\bot)))$ so the plan will move
the disks from peg 2 to peg 3. 

A closely related problem is to have a number of independent Towers of Hanoi; these
can be represented by a term $g(t_1, t_2, \cdots, t_m)$ where $t_i$ is a term
representing the state of the $i^{th}$ Tower of Hanoi problem.

\subsection{Crossing a river}

Here the problem is to arrange a group of people in some specified way on both banks
of a river.  There is a bridge across the river with three intermediate locations.  If two
people meet going opposite directions, they cannot cross each other, so one has to back
up.  This can be represented by a term $f(t_1,t_2,t_3,t_4,t_5)$ where the $t_i$ are lists
of people that are at the given location.  $t_1$ is the left bank, $t_5$ is the right bank,
and $t_2, t_3$, and $t_4$ are the intermediate locations on the bridge.  $t_2$, $t_3$,
and $t_4$ can have at most one person at a time.

The actions are to sort the lists in $t_1$ and $t_5$ by exchanging adjacent elements (which is an action in $E_{\sigma}$, and to move an element at the head of the list from $t_i$
to $t_{i+1}$, $1 \le i \le 4$ and from $t_{i+1}$ to $t_i$, $1 \le i \le 4$.  A person can be moved to $t_2, t_3$, or $t_4$ only if these
locations are empty, but a person can be moved to $t_1$ or $t_5$ at any time.

A sample configuration would be $f(g(2,g(3,\bot)),\bot,g(4,\bot),g(5,\bot),g(1,\bot))$ indicating that
persons 2 and 3 are on the left bank, person 1 is on the right bank, and persons 4 and 5
are crossing.

With a lexicographic ordering by the lengths of the lists $t_i$, given a problem of
transforming $f(s_1,\cdots,s_5)$ to $f(t_1,\cdots,t_5)$, completion would result in
a plan to move everyone to the right bank and then move them back to where they should
be in the lists $t_i$.  Again, it might be possible to apply local optimizations to this plan to
make it more efficient.

\subsection{Blocks world}

In this domain, there are a fixed number of positions, each having a tower of blocks.  The
blocks can be piled in any order, and at any time the top block in any tower can be moved
on top of any other tower.  Then one wants a plan to transform some specified starting
state to a goal state.

This can be represented by a list $f((s_1,x_1),f((s_2,x_2), \cdots, f((s_n,x_n), \bot) \cdots ))$ where
the $s_i$ are lists of blocks and the $x_i$ are their locations.  The top block on a list $s_i$
appears first, then the blocks underneath it. The list of blocks is represented by
$g(b_1,g(b_2,\cdots))$ where $b_1$ is the top block and $b_2$ is next under it, and so
on.  The actions include permuting the lists of
blocks:  $f(t_1,f(t_2, x )) \rightarrow f(t_2,f(t_1, x))$ to exchange adjacent
towers of blocks (an action in $E_{\sigma}$) and an action
$f((g(b_1,s_1),x_1),f((s_2,x_2),z)) \rightarrow f((s_1,x_1),f((g(b_1,s_2),x_2),z))$ to
move the top block $b_1$ from the tower $g(b_1,s_1)$ at $x_1$ to the tower $s_2$ at $x_2$.

For this example, it's not clear what kind
of a plan the completion approach would generate, or what would be a suitable ordering.
However, if the ordering is lexicographic by the sizes of the towers, then the plan will
rewrite both starting and goal state terms to terms $t_1$ and $t_2$ with all blocks in one large tower.
Then the blocks in these towers will be permuted to go from $t_1$ to $t_2$; this will
be done using rules generated during completion.

\section{Rewriting on Subterms}
We show formally that sets $R_{\cal U}$ of rewrite rules satisfying the conditions specified do exist for many theories $\cal U$, and
give an idea how they can be constructed.
First, without using rewriting on proper subterms and using only ground terms, it is
always possible to construct $R_{\cal U}$. The following definition of $R_{\cal U}$ only includes rules
that rewrite an entire term in ${\cal T}$.

\begin{definition}
Let ${\cal U}$ be any situation calculus satisfying the fluent dependence condition of
definition \ref{fluent.dependence.condition}, let ${\cal T}$ be a set of ground terms, and let $\sigma$ be a function satisfying Equation \ref{term.set.definition}.  Let $A^0_{\sigma}$
be a set of rules $\alpha_L \rightarrow \alpha_R$ where $\alpha_L, \alpha_R \in {\cal T}$ are terms such that there are states $s,t$ in
${\cal U}$ and an action $a \in {\cal A}$ such that $\sigma(\alpha_L) = s$ and
$\sigma(\alpha_R) = t$ and $t = do(a,s)$.  One such rule is chosen for each $\alpha_L$ and
each action $a$.  Let $E_{\sigma}$ be some set of rules
satisfying Equations \ref{state.closure.equation} and \ref{rearrangements.to.states.equation}.  Let $R^0_{\cal U}$  be $A^0_{\sigma} \cup
E_{\sigma}$.
\end{definition}

\begin{theorem}
$R^0_{\cal U}$ represents the theory ${\cal U}$.
\end{theorem}

\begin{proof}
This is straightforward from the definitions.  Equation \ref{rule.to.action.equation} is
a direct consequence of the construction.  Equation \ref{do.rewrite.equation} follows
because for each $\alpha_L$ and each action $a$, one $\alpha_R$ is chosen such that
$\sigma(\alpha_R) = do(a,\sigma(\alpha_L))$ but all such $\alpha_R$ are $E_{\sigma}$
equivalent by Equation \ref{rearrangements.to.states.equation} and the fluent
dependence condition of Definition \ref{fluent.dependence.condition}.
\end{proof}

This construction however generally produces a huge term-rewriting system; the number of rules is at least as large as the number of states in ${\cal S}$.  It is helpful to understand in general how this number can be reduced.

The construction of $R_{\cal U}$ can be made more effective if one permits rewriting on proper subterms of terms in ${\cal T}$,
as shown in the examples.  This corresponds to encoding frame axioms of ${\cal U}$
because the fluents that depend on the term structure outside of the rewritten subterm
will not change.  It also may help to allow $R_{\cal U}$ to contain non-ground rules.
These possibilities are considered in the following
results.  The first two definitions are interesting but are not used in the succeeding
results.

\begin{definition}
Given a context $u$ over $T[F,X]$, ${\cal T}|_u$ is the set of terms $t$ such that
$u[t] \in {\cal T}$.
\end{definition}

\begin{definition}
A term $t$ in $T[F,X]$ is {\em uniform} for ${\cal T}$ if for all contexts $u$ and $u'$ over $T[F,X]$,
if $u[t] \in {\cal T}$ and $u'[t] \in {\cal T}$ then ${\cal T}|_u = {\cal T}|_{u'}$.
\end{definition}

The following definition essentially specifies which rewrite rules can be used to
represent actions in ${\cal U}$.

\begin{definition}
Suppose ${\cal U}$ is a situation calculus satisfying the fluent dependence condition of
Definition \ref{fluent.dependence.condition}, suppose ${\cal T}$ is a set of ground terms, and let $\sigma$ be a function satisfying Equation \ref{term.set.definition}.  
Then a term $t$ in $T[F,X]$ is an {\em action support} and the
rule $t \rightarrow t'$ is an {\em action rule} if for all contexts
$u$ over $T[F,X]$, if $u[t] \in {\cal T}$ then there is a term $t'$ with $u[t'] \in
{\cal T}$ and an action $a \in {\cal A}$ such that $\sigma(u[t']) = do(a,\sigma(u[t]))$.
\end{definition}

The next definition attaches actions and terms to action supports and action rules.  More
than one action and term can be attached to the same action rule.

\begin{definition}
Suppose ${\cal U}$ is a situation calculus satisfying the fluent dependence condition of
Definition \ref{fluent.dependence.condition}, suppose ${\cal T}$ is a set of ground terms, and let $\sigma$ be a function satisfying Equation \ref{term.set.definition}.
Suppose $u[t] \in {\cal T}$ and $t$ is an action support. Suppose $t'$ is a term in
$T[F,X]$ such that $t \rightarrow t'$ is an action rule and $\sigma(u[t']) =
do(a,\sigma(u[t]))$.  Then $t$ is an {\em action support for action $a$ and term $u[t]$}
and $t \rightarrow t'$ is an {\em action rule for action $a$ and term $u[t]$}.
\end{definition}

The following definition specifies a way of choosing rewrite rules for $R_{\cal U}$ that
will tend to choose rules that rewrite small subterms of terms in ${\cal T}$.

\begin{definition}
Suppose ${\cal U}$ is a situation calculus satisfying the fluent dependence condition of
Definition \ref{fluent.dependence.condition}, suppose ${\cal T}$ is a set of ground terms, and let $\sigma$ be a function satisfying Equation \ref{term.set.definition}.  Let $A^1_{\sigma}$
be a set $\alpha_L \rightarrow \alpha_R$ of action rules for actions $a \in {\cal A}$
and terms $u \in {\cal T}$ such that one such rule is chosen for each action $a$ and each term
$u$.  The rule that is chosen is one such that $\alpha_L$ is
a minimal size action support for the action $a$ and the term $u$.  Let $E_{\sigma}$
be some set of rules
satisfying Equations \ref{state.closure.equation} and \ref{rearrangements.to.states.equation}.  Let $R^1_{\cal U}$  be $A^1_{\sigma} \cup
E_{\sigma}$.
\end{definition}

\begin{theorem}
$R^1_{\cal U}$ represents the theory ${\cal U}$.
\end{theorem}

\begin{proof}
Again, this is straightforward from the definitions.  This set of rules may be much
smaller than $R^0_{\cal U}$ because of the use of minimal action support terms.
The choices of which action support term to use and which term $\alpha_R$ to use
do not matter, because of the fluent dependence condition of Definition \ref{fluent.dependence.condition}.
\end{proof}

In this respect, the switches example is interesting because a given rewrite rule can
express more than one action.  The rule $\mbox{off} \rightarrow \mbox{on}$, for example, can
express turning on any of the switches, depending on where it is used in a term in ${\cal T}$.

The next definition lifts $R^1_{\cal U}$ to a possibly non-ground term rewriting system
$R^2_{\cal U}$
that may be much more compact.

\begin{definition}
Suppose ${\cal U}$ is a situation calculus satisfying the fluent dependence condition of
Definition \ref{fluent.dependence.condition}, let ${\cal T}$ be a set of ground terms, and let $\sigma$ be a function satisfying Equation \ref{term.set.definition}.  
Let $A^2_{\sigma}$ be a set of possibly non-ground rules $\alpha_L \rightarrow \alpha_R$ such that
for all ground instances $\alpha_L \beta$ of $\alpha_L$ with $\alpha_L \beta \in {\cal T}$,
$\alpha_R \beta \in {\cal T}$ also and $\alpha_L\beta \rightarrow \alpha_R\beta \in
A^1_{\cal U}$.  Let $E_{\sigma}$
be some set of rules
satisfying Equations \ref{state.closure.equation} and \ref{rearrangements.to.states.equation}.  Let $R^2_{\cal U}$  be $A^2_{\sigma} \cup
E_{\sigma}$.
\end{definition}

\begin{theorem}
$R^2_{\cal U}$ represents the theory ${\cal U}$.
\end{theorem}

\begin{proof}
Again, this is straightforward from the definitions.  The rewrite relation for
$R^2_{\cal U}$ on ${\cal T}$ is the same as that for $R^1_{\cal U}$. The set
$R^2_{\cal U}$ of rules may be much
smaller even than $R^1_{\cal U}$ because of the use of non-ground rules.
\end{proof}

For the first two examples, $E_{\sigma}$ is empty.  For the river crossing example,
$E_{\sigma}$ consists of rules that exchange adjacent elements of the lists in $t_1$ and
$t_5$.  For the blocks world example, $E_{\sigma}$ consists of the rules exchanging
adjacent elements of the list of towers of blocks.
The rules $R_{\cal U}$ given for the Tower of Hanoi problem are $R^1_{\cal U}$.
The rules given for the switches problem are also $R^1_{\cal U}$.  For both of these
examples, nothing can be gained by using non-ground rules.  For the crossing the
river example, using subterms helps to sort the lists in $t_1$ and $t_5$ and the use of non-ground terms
also helps in general so this is an example where $R^2_{\cal U}$ is better than 
$R^1_{\cal U}$.  For the blocks world example, both the use of subterms and the use of non-ground
tems contribute to making $R^2_{\cal U}$ more concise than $R^1_{\cal U}$ and
$R^0_{\cal U}$.

Now we examine which properties a term needs in order to be an action
support term, in order to gain more understanding of the construction of $R_{\cal U}$.

\begin{definition}
\label{F limited F expressive definition}
Suppose ${\cal F}'$ is a subset of ${\cal F}$.  A subset ${\cal T}'$ of ${\cal T}$ is
{\em ${\cal F}'$-limited} if for all pairs $t_1,t_2 \in {\cal T}'$, and all contexts $u$, if $u[t_1]$ and $u[t_2]$ are
two terms in ${\cal T}$, then $\hat{\Phi}(p,u[t_1])=\hat{\Phi}(p,u[t_2])$ for all
$p \in {\cal F} \setminus {\cal F}'$.  (The choice of $t_1$ or $t_2$ can only influence
fluents in ${\cal F}'$.)  The subset ${\cal T}'$ of ${\cal T}$ is
{\em ${\cal F}'$-expressive} if for all states $s \in {\cal S}$ such that $\hat{\Phi}(p,u[t_1])=\Phi(p,s)$ for all
$p \in {\cal F} \setminus {\cal F}'$, there is a term $t_2 \in {\cal T}'$
such that $\sigma(u[t_2])=s$.  (The choice of a term in the set can produce a
full combination of fluents for that position in $u$, in some sense.)
\end{definition}

\begin{definition}
\label{a limited definition}
An action $a \in {\cal A}$ is {\em ${\cal F}$-limited} if for all states $s \in {\cal S}$,
for all fluents $p \in {\cal F} \setminus {\cal F}'$, $\Phi(p,s) = \Phi(p,do(a,s))$ 
($a$ does not change any fluents outside of ${\cal F}'$)
and
if for all states $s_1, s_2 \in {\cal S}$, if $\Phi(p,s_1) = \Phi(p,s_2)$ for all $p \in {\cal F}'$
then $\Phi(p,do(a,s_1)) = \Phi(p,do(a,s_2))$ for all $p \in {\cal F}'$. (The action $a$ does not
depend on any fluents outside of ${\cal F}'$).
\end{definition}

The following result gives sufficient conditions for a rewrite rule $t \rightarrow t'$ to exist that represents
an action, at least when it is applied to the term $t$.

\begin{theorem}
If $u[t]$ is a term in ${\cal T}$, the set of $t'$ such that $u[t'] \in {\cal T}$ is
${\cal F}'$ limited and ${\cal F}'$ expressive for some ${\cal F}' \subseteq {\cal F}$,
then for every ${\cal F'}$ limited action $a \in {\cal A}$ and every term $t$ such that $u[t] \in {\cal T}$,
there is a term $t'$ such that $u[t'] \in {\cal T}$ also and $\sigma(u[t']) =
do(a,\sigma(u[t]))$.
\end{theorem}

\begin{proof}
The ${\cal F}'$ expressive condition guarantees that such a term $t'$ exists. If the action
a is not ${\cal F}'$ limited then it would have to change some of the term structure outside
the occurrence of $t$.
\end{proof}

The following definition and theorem give a weak necessary condition for a rewrite rule on
a subterm to exist that expresses an action.

\begin{definition}
A set of terms is {\em weakly ${\cal F}'$ expressive} for ${\cal F}' \subseteq {\cal F}$
if there is a context $u$ with $u[t] \in {\cal T}$ for some term $t$ and there are
at least two terms $t_1$ and $t_2$ such that $u[t_1]$ and $u[t_2]$ are both in ${\cal T}$,
$\hat{\Phi}(p,u[t_1]) = \hat{\Phi}(p,u[t_2])$ for all $p \in {\cal F} \setminus {\cal F}'$,
but $\hat{\Phi}(p,u[t_1]) \ne \hat{\Phi}(p,u[t_2])$ for some $p \in {\cal F}'$.
\end{definition}

\begin{theorem}
Suppose $a \in {\cal A}$, $a$ is ${\cal F}'$ limited for some ${\cal F}' \subseteq {\cal F}$,
$u[t] \in {\cal T}$ for some term $u[t]$, there is a term
$t'$ such that $\sigma(u[t']) = do(a, \sigma(u[t])$, and for some $p \in {\cal F}'$,
$\hat{\Phi}(p,u[t']) \ne \hat{\Phi}(p,u[t])$.  Then the set of terms $t$ such
that $u[t] \in {\cal T}$ is weakly ${\cal F}'$ expressive.
\end{theorem}

\begin{proof}
For $t_1$ and $t_2$ one takes the terms $t$ and $t'$ of the theorem.  These terms
do not agree on all fluents because this is stated in the theorem, and it seems reasonable
for an action to change at least one fluent in many cases.  The terms $t$ and $t'$
agree on all fluents not in ${\cal F}'$ because the action $a$ is ${\cal F}'$ limited.
\end{proof}

\section{Conclusion}
After a brief survey of the situation calculus, term-rewriting systems are introduced and
situation calculus concepts are presented.  Next an approach to encoding the situation
calculus by term rewriting is presented.  A general result for planning using this
approach is given.  Four examples illustrate the properties of this approach.  General methods for constructing term rewriting systems embodying this approach are given,
and special attention is given to rewriting on subterms, which corresponds to encoding
frame axioms in the underlying theory, and lifting the rewrite rules to non-ground rules.
Finally, some results are given about sufficient conditions and a necessary condition for a rewrite
rule to exist that represents an action on a specific term.

\bibliography{paper}

\end{document}